\documentclass{article}
\usepackage{iclr2025_conference,times}

\usepackage{amsmath,amsfonts,bm}

\def\eqref#1{equation~\ref{#1}}

\def\1{\bm{1}}

\DeclareMathAlphabet{\mathsfit}{\encodingdefault}{\sfdefault}{m}{sl}
\SetMathAlphabet{\mathsfit}{bold}{\encodingdefault}{\sfdefault}{bx}{n}

\usepackage{hyperref}
\usepackage{url}
\usepackage{makecell}
\usepackage{multirow}
\usepackage{paralist}

\PassOptionsToPackage{numbers}{natbib}

\usepackage[utf8]{inputenc} 
\usepackage[T1]{fontenc}    
\usepackage{booktabs}       
\usepackage{amsfonts}       
\usepackage{nicefrac}       
\usepackage{microtype}      
\usepackage{xcolor}         

\usepackage{graphicx}
\usepackage{subfigure}
\usepackage{centernot}
\usepackage{multirow}
\usepackage{multicol}
\usepackage{adjustbox}
\usepackage{caption}
\usepackage{subcaption}
\usepackage{float}
\usepackage{pgfkeys}
\usepackage{array}
\usepackage{tikz}
\usepackage[inline]{enumitem}
\usepackage{enumitem}
\usepackage{rotating}
\usepackage{pdflscape}
\usepackage{bbm}
\usepackage{mathrsfs}
\usepackage{wrapfig}
\usepackage{soul}
\usepackage{appendix}
\usepackage{etoolbox}
\usepackage{blkarray}
\usepackage{comment}
\usepackage{scalerel}

\usepackage{multirow}
\usepackage{enumitem}

\usepackage{amsmath}
\usepackage{amssymb}
\usepackage{mathtools}
\usepackage{amsthm}
\usepackage{amsfonts}
\usepackage{thm-restate}
\usepackage{tikz}
\usepackage{bbm}
\usepackage{amsfonts}       
\usepackage{mathrsfs}
\usepackage{xcolor}         

\theoremstyle{plain}
\newtheorem{theorem}{Theorem}[section]

\newtheorem{corollary}[theorem]{Corollary}
\theoremstyle{definition}
\newtheorem{definition}[theorem]{Definition}

\theoremstyle{remark}

\theoremstyle{definition}

\newcommand{\W}{W}

\newcommand{\Fc}{\mathcal F}
\newcommand{\Gc}{\mathcal G}
\newcommand{\Sc}{\mathcal S}

\newcommand{\Xc}{\mathcal X }
\newcommand{\Yc}{\mathcal Y}
\newcommand{\Zc}{\mathcal Z}

\newcommand{\Wc}{\mathcal W}
\newcommand{\Tc}{\mathcal T}
\newcommand{\Nb}{\mathbb N}
\newcommand{\Rb}{\mathbb R}

\newcommand{\Fb}{\mathcal F}

\newcommand{\sd}[1]{{{\footnotesize±}{\scriptsize#1}}}

\newcommand{\fa}{\lambda}

\newcommand{\agg}{\mathsf{agg}}
\newcommand{\upd}{\mathsf{upd}}
\newcommand{\readout}{\mathsf{readout}}

\newcommand{\Hom}{\mathsf{Hom}}

\newcommand{\Lip}{\mathsf{Lip}}

\newcommand{\WL}{\mathsf{1}\text{-}\text{WL}}
\newcommand{\WLF}[1]{#1\text{-}\text{WL}}

\usetikzlibrary{shapes.geometric}
\tikzset{gon/.style={name=tmp,regular polygon,regular polygon sides=#1,minimum
size=10pt,inner sep=0pt},flag connection/.style={-},
polygon side/.style args={#1--#2}{
insert path={(tmp.corner #1) edge[flag connection] (tmp.corner #2)}}}
\newcommand{\FlagGraph}[3][false]{
\ifnum#2=1%
\tikz[baseline=(tmp1)]{\node[circle,inner sep=1pt,minimum size=0.8mm,draw=black, fill=\ifbool{#1}{blue}{white}] (tmp1) at (0,0){};}
\else%
\ifnum#2=2%
 \tikz[baseline=(tmp1)]{
 \node[circle,inner sep=1pt,minimum size=0.8mm,draw=black, fill=\ifbool{#1}{blue}{white}] (tmp1) at (0,0){};
 \node[circle,inner sep=1pt,draw=black] (tmp2) at (10pt,0){};
 \ifx#3\empty%
 \else
 \draw (tmp1) edge[flag connection] (tmp2);
  \fi}
\else
\tikz[baseline=(tmp.south)]{\node[gon=#2,inner sep=0pt,minimum size=5mm]{};
\draw[draw=black, fill=\ifbool{#1}{blue}{white}] (tmp.corner 1) circle (1.5pt);
\foreach \X in {2,...,#2}{\draw (tmp.corner \X) circle (1.5pt);}
\draw[polygon side/.list={#3}]}
\fi
\fi
}

\makeatletter
\def\myl{\mathopen\mybig}
\def\myr{\mathclose\mybig}
\def\mybigx#1{\dimen@#1\relax
\mathchoice
{\vbox to \dimen@{}}%
{\vbox to \dimen@{}}%
{\vbox to .7\dimen@{}}%
{\vbox to .5\dimen@{}}}%

\def\mybig#1{{\hbox{$\left#1\mybigx{0.81em}\right.\n@space$}}}
\makeatother

\newcommand{\graphK}{%
\raisebox{-0.2\height}{\resizebox{0.4cm}{!}{%
\begin{tikzpicture}[
    node distance={30mm}, 
    thick, 
    rotate=30, 
    main/.style = {draw, circle, minimum size=12mm, line width=4pt, text opacity=0}, 
    every edge/.style={line width=3pt},  
    blacknode/.style = {draw, circle, fill=black, minimum size=12mm, line width=4pt, text opacity=0}
    ]

\node[main] (3) {3}; 

\node[main] (1) [left of=3] {1}; 
\node[main] (2) [above of=3] {2}; 
\node[blacknode] (4) [above of=1] {4}; 

\draw[line width=4pt] (1) -- (2); 
\draw[line width=4pt] (1) -- (3); 
\draw[line width=4pt] (1) -- (4); 
\draw[line width=4pt] (2) -- (3); 
\draw[line width=4pt] (2) -- (4); 
\draw[line width=4pt] (3) -- (4); 

\end{tikzpicture}%
}}%
}

\newcommand{\graphC}{%
\raisebox{-0.2\height}{\resizebox{0.4cm}{!}{%
\begin{tikzpicture}[
    node distance={30mm}, 
    thick, 
    main/.style = {draw, circle, minimum size=12mm, line width=4pt, text opacity=0}, 
    every edge/.style={line width=3pt},  
    blacknode/.style = {draw, circle, fill=black, minimum size=12mm, line width=4pt, text opacity=0}
    ]

\node[main] (3) {3}; 

\node[main] (1) [left of=3] {1}; 
\node[main] (2) [above of=3] {2}; 
\node[blacknode] (4) [above of=1] {4}; 

\draw[line width=4pt] (1) -- (3); 
\draw[line width=4pt] (1) -- (4); 
\draw[line width=4pt] (2) -- (3); 
\draw[line width=4pt] (2) -- (4); 

\end{tikzpicture}%
}}%
}

\newcommand{\graphGM}{%
\raisebox{-0.5\height}{\resizebox{0.85cm}{!}{%
\begin{tikzpicture}[
    node distance={30mm}, 
    thick, 
    main/.style = {draw, circle, minimum size=12mm, line width=3pt, text opacity=0}, 
    every edge/.style={line width=3pt}  
    ]

\node[main] (3) {3}; 

\node[main] (1) [left of=3] {1}; 
\node[main] (2) [above of=3] {2}; 
\node[main] (4) [above of=1] {4}; 

\draw[line width=3pt] (1) -- (2); 
\draw[line width=3pt] (1) -- (3); 
\draw[line width=3pt] (1) -- (4); 
\draw[line width=3pt] (2) -- (3); 
\draw[line width=3pt] (2) -- (4); 
\draw[line width=3pt] (3) -- (4); 

\node[main] (5) [below left of=3, shift={(0.7,-0.6)}] {5}; 
\node[main] (6) [below right of=3, shift={(0.5,0.1)}] {6}; 
\node[main] (7) [below of=3, shift={(1.1,-1.9)}] {7}; 

\draw[line width=3pt] (3) -- (5); 
\draw[line width=3pt] (3) -- (6); 
\draw[line width=3pt] (3) -- (7); 
\draw[line width=3pt] (5) -- (6); 
\draw[line width=3pt] (5) -- (7); 
\draw[line width=3pt] (6) -- (7); 

\draw[line width=3pt] (1) -- (5);

\end{tikzpicture}%
}}%
}

\newcommand{\graphGGM}{%
\raisebox{-0.5\height}{\resizebox{0.8cm}{!}{%
\begin{tikzpicture}[
    node distance={30mm}, 
    thick, 
    main/.style = {draw, circle, minimum size=12mm, line width=3pt, text opacity=0}, 
    every edge/.style={line width=3pt}  
    ]

\node[main] (3) {3}; 

\node[main] (1) [left of=3] {1}; 
\node[main] (2) [above of=3] {2}; 
\node[main] (4) [above of=1] {4}; 

\draw[line width=3pt] (1) -- (2); 
\draw[line width=3pt] (1) -- (3); 
\draw[line width=3pt] (1) -- (4); 
\draw[line width=3pt] (2) -- (3); 
\draw[line width=3pt] (2) -- (4); 
\draw[line width=3pt] (3) -- (4); 

\node[main] (5) [below left of=3, shift={(0.7,-0.5)}] {5}; 
\node[main] (6) [below right of=3, shift={(0,-0)}] {6}; 
\node[main] (7) [below of=6, shift={(0.3,0.1)}] {7}; 
\node[main] (8) [below of=5, shift={(0.4,0.1)}] {8}; 

\draw[line width=3pt] (3) -- (5); 
\draw[line width=3pt] (3) -- (6); 
\draw[line width=3pt] (5) -- (6); 
\draw[line width=3pt] (5) -- (7); 
\draw[line width=3pt] (6) -- (7); 

\draw[line width=3pt] (5) -- (8); 
\draw[line width=3pt] (7) -- (8); 

\draw[line width=3pt] (1) -- (5);

\end{tikzpicture}%
}}%
}

\newcommand{\graphHM}{%
\raisebox{-0.5\height}{\resizebox{0.85cm}{!}{%
\begin{tikzpicture}[
    node distance={30mm}, 
    thick, 
    main/.style = {draw, circle, minimum size=12mm, line width=3pt, text opacity=0}, 
    every edge/.style={line width=3pt}  
    ]

\node[main] (3) {3}; 

\node[main] (1) [left of=3] {1}; 
\node[main] (2) [above of=3] {2}; 
\node[main] (4) [above of=1] {4}; 

\draw[line width=3pt] (1) -- (2); 
\draw[line width=3pt] (1) -- (3); 
\draw[line width=3pt] (1) -- (4); 
\draw[line width=3pt] (2) -- (3); 
\draw[line width=3pt] (2) -- (4); 
\draw[line width=3pt] (3) -- (4); 

\node[main] (5) [below left of=3, shift={(0.7,-0.6)}] {5}; 
\node[main] (6) [below right of=3, shift={(0.5,0.1)}] {6}; 
\node[main] (7) [below of=3, shift={(1.1,-1.9)}] {7}; 

\draw[line width=3pt] (3) -- (5); 
\draw[line width=3pt] (3) -- (6); 
\draw[line width=3pt] (3) -- (7); 
\draw[line width=3pt] (5) -- (6); 
\draw[line width=3pt] (5) -- (7); 
\draw[line width=3pt] (6) -- (7); 

\end{tikzpicture}%
}}%
}

\newcommand{\graphhM}{%
\raisebox{-0.5\height}{\resizebox{0.85cm}{!}{%
\begin{tikzpicture}[
    node distance={30mm}, 
    thick, 
    main/.style = {draw, circle, minimum size=12mm, line width=3pt, text opacity=1}, 
    every edge/.style={line width=3pt},  
    every loop/.style={min distance=10mm, out=135, in=45}  
    ]

\node[main] (3) {3}; 

\node[main] (1) [left of=3] {1}; 
\node[main] (2) [above of=3] {2}; 
\node[main] (4) [above of=1] {4}; 

\draw[line width=3pt] (1) -- (2); 
\draw[line width=3pt] (1) -- (3); 
\draw[line width=3pt] (1) -- (4); 
\draw[line width=3pt] (2) -- (3); 
\draw[line width=3pt] (2) -- (4); 
\draw[line width=3pt] (3) -- (4); 

\node[main] (5) [below of=1] {5}; 
\node[main] (6) [below of=3] {6}; 

\draw[line width=3pt] (3) -- (5); 
\draw[line width=3pt] (3) -- (6); 
\draw[line width=3pt] (5) -- (6); 

\draw[line width=3pt] (1) -- (5);
\draw[line width=3pt] (1) -- (6);

\node[main] (7) [above of=4, shift={(1.3,-0.5)}] {7}; 

\draw[line width=3pt] (7) -- (4); 
\draw[line width=3pt] (7) -- (2); 

\draw[line width=3pt, bend right=50] (7) to (5);
\draw[line width=3pt, bend left=50] (2) to (6);

\end{tikzpicture}%
}}%
}

\usepackage[capitalize,noabbrev]{cleveref}

\definecolor{RoyalBlue}{rgb}{0,0,0.8}
\hypersetup{colorlinks, 
citecolor=RoyalBlue, 
linkcolor=RoyalBlue,
urlcolor=brown}

\title{
Towards Bridging Generalization and Expressivity of Graph Neural Networks}

\author{%
    Shouheng Li\textsuperscript{\rm 1,\rm 4}, 
    Floris Geerts\textsuperscript{\rm 2}, 
    Dongwoo Kim\textsuperscript{\rm 3}, 
    Qing Wang\textsuperscript{\rm 1} \\
    \textsuperscript{\rm 1} School of Computing, Australian National University, Australia\\
    \textsuperscript{\rm 2} Department of Computer Science, University of Antwerp   , Belgium\\
    \textsuperscript{\rm 3} CSE \& GSAI, POSTECH, South Korea\\
    \textsuperscript{\rm 4} Data61, CSIRO, Australia\\
    \texttt{shouheng.li@anu.edu.au, floris.geerts@uantwerp.be}\\
    \texttt{dongwoo.kim@postech.ac.kr, qing.wang@anu.edu.au}
}

\iclrfinalcopy 
\begin{document}

\maketitle

\begin{abstract} 
Expressivity and generalization are two critical aspects of graph neural networks (GNNs). While significant progress has been made in studying the expressivity of GNNs, much less is known about their generalization capabilities, particularly when dealing with the inherent complexity of graph-structured data.
In this work, we address the intricate relationship between expressivity and generalization in GNNs. Theoretical studies conjecture a trade-off between the two: highly expressive models risk overfitting, while those focused on generalization may sacrifice expressivity. However, empirical evidence often contradicts this assumption, with expressive GNNs frequently demonstrating strong generalization. We explore this contradiction by introducing a novel framework that connects GNN generalization to the variance in graph structures they can capture. This leads us to propose a $k$-variance margin-based generalization bound that characterizes the structural properties of graph embeddings in terms of their upper-bounded expressive power. Our analysis does not rely on specific GNN architectures, making it broadly applicable across GNN models. We further uncover a trade-off between intra-class concentration and inter-class separation, both of which are crucial for effective generalization. Through case studies and experiments on real-world datasets, we demonstrate that our theoretical findings align with empirical results, offering a deeper understanding of how expressivity can enhance GNN generalization. 
\end{abstract}

\section{Introduction}
\label{sec:intro}
Graph Neural Networks (GNNs)~\citep{scarselli2008graph} have become pivotal in modern machine learning, anchored in two main pillars: \textit{expressivity} and \textit{generalization}. Expressivity refers to a GNN's capacity to distinguish between diverse graph structures, thereby determining the scope of problems it can address~\citep{xu2018powerful,morris2019weisfeiler}. Highly expressive GNNs can capture intricate dependencies, essential for tasks like molecular property prediction~\citep{GilmerSRVD17}, drug discovery~\citep{gaudelet2021utilizing}, and protein-protein interaction prediction~\citep{zitnik2018modeling}, where minor structural variations have significant implications. Generalization, on the other hand, reflects a GNN’s ability to transfer learned knowledge to unseen graphs. Given the diversity in graph structures, sizes, and complexities, GNNs that generalize well maintain consistent performance across varying datasets. Together, these properties enable GNNs to model complex graph structures while remaining effective across new, unseen data, making them invaluable for graph-based analysis.

Theoretically, a trade-off is expected between expressivity and generalization: highly expressive models can capture complex graph structures but may overfit and generalize poorly without proper regularization. Conversely, models focused on generalization often sacrifice some expressivity to perform better across diverse, unseen graph structures. Recent work indeed shows a strong correlation between a GNN’s VC dimension and its ability to distinguish non-isomorphic graphs~\citep{morris23meet}. \emph{A more nuanced theoretical analysis is needed}, however.
Indeed, \emph{empirical evidence frequently contradicts the above view.}
Highly expressive models often exhibit strong generalization performance in practice~\citep{Bouritsas2023-hj,Wang_undated-iz}. In the restricted context of linear separability, margin-based bounds offer partial alignment between theory and practice~\citep{franksweisfeiler}, yet our broader understanding of how expressivity influences generalization remains incomplete. This raises two key questions: (i) \emph{How does the structured nature of graphs affect GNN generalization?} (ii) \emph{How does a GNN’s expressivity influence its ability to generalize across tasks and unseen data?} Addressing these questions is vital for advancing GNN applications in real-world scenarios.

\noindent\textbf{Present work.}
Building on the foundational work of \citet{Chuang2021-ik}, we explore how the \emph{concentration} and \emph{separation} of \emph{learned features}, key factors in multiclass classification generalization, translate to graph-based models. Their bound, derived from $k$-variance~\citep{solomonGN22} and the expected optimal transport cost between two random subsets of the training distribution, motivates our adaptation to graph embeddings. Leveraging these insights, we extend their framework to capture the structural properties of graph embedding distributions and contribute the following:
\begin{itemize}
    \item For arbitrary graph encoders, including GNNs, we show that their generalization can be bounded in terms of the generalization bound of any more expressive graph encoder. This allows capturing structural properties of graph embedding distributions with respect to their bounding encoders.
    \item Under certain margin conditions, we demonstrate that the downstream classifier generalizes well if (1) embeddings within a class are well-clustered and (2) classes are separable in the embedding space in the Wasserstein sense, extending \citet{Chuang2021-ik}'s results to graphs.
    \item On the real-world PROTEINS dataset~\citep{Morris+2020}, we empirically show how a more expressive model influences generalization by measuring variance in graph embedding distributions.
    \item We apply the empirical sample-based bound of \citet{Chuang2021-ik} to graph classification tasks, verifying that empirical findings align with our theoretical insights, thus demonstrating the applicability of our approach to predict generalization.
\end{itemize}

Our results offer a flexible framework for analyzing generalization properties of complex graph encoders via simpler encoders, such as those based on $\WL$, its higher-order variants $\WLF{k}$ \citep{cai1992optimal,grohe2017descriptive},  homomorphism counts~\citep{zhang2024-quantitative}, or $\WLF{\Fb}$ \citep{Barcelo2021-rs}, provided they upper-bound the encoders under consideration. \emph{Overall, we present a versatile tool for evaluating whether increased expressiveness improves or worsens generalization.}

\section{Related work}
Regarding generalisation of GNNs,
\citet{ScarselliTH18} utilize VC dimension to study the generalization of an older GNN architecture, distinct from modern MPNNs~\citep{GilmerSRVD17}. \citet{garg2020radmacher} show that the Rademacher complexity of simple GNNs depends on maximum degree, layer count, and parameter norms, while \citet{liao21pacbayes} develop PAC-Bayesian bounds relying on node degree and spectral norms; see \citet{karczewski2024on} for extensions. Improved bounds using the largest singular value of the diffusion matrix are proposed by \citet{ju2023diffusion}. Transductive PAC-Bayesian bounds for knowledge graphs are discussed by \citet{lee2024}. Random graph models are leveraged by \citet{maskey2022random}, who show GNN generalization improves with larger graphs. Connections between VC-dimension and the 1-WL algorithm are made by \citet{morris23meet}, who bound it by the number of 1-WL colors. \citet{levie2023a} provide bounds based on covering numbers and specialized graph metrics.

For GCNs, \citet{verma2019stability} derive generalization bounds using algorithmic stability, with \citet{zhang2020fast} focusing on single-layer GCNs and accelerated gradient descent. \citet{zhou2021generalization} extend this to multi-layer GCNs, showing that generalization gaps increase with depth. Similarly, \citet{CongRM21} highlight this trend in deeper GNNs and propose detaching weight matrices to improve generalization. Further analyses of transductive Rademacher complexity using stochastic block models are offered by \citet{oono2020optimization,esser2021sometimes}. \citet{tang2023towards} establish bounds involving node degree, training iterations, and Lipschitz constants, while \citet{Li2022sampling} study topology sampling and its impact on generalization. Lastly, \citet{franksweisfeiler} explore margin-based bounds. 

Moving to the expressivity of GNNs, MPNNs' expressivity is bounded by $\WL$~\citep{xu2018powerful,morris2019weisfeiler}, showing the need for more expressive methods. Many such models have been put forward. For example, the $\Fb$-MPNNs~\citep{Barcelo2021-rs} enhance expressivity via homomorphism counts, similar to \citet{Bouritsas2023-hj}. Homomorphism counts have become a popular mechanism in graph learning \citep{Nguyen2020hom,Welke_2023-xn,zhang2024-quantitative,JinBCL24,lanzinger2024on}, and will be central to our analysis. Additional discussion on related work can be found  in \cref{sec:apprelwork}.

\section{Preliminaries}\label{sec:prel}

\paragraph{Graphs and homomorphisms.}
We begin by considering undirected graphs \(G = (V_G, E_G)\), where \(V_G\) represents the set of \emph{vertices} and \(E_G \subseteq V_G \times V_G\) forms the \emph{edge} set, a symmetric relation. For any vertex \(v \in V_G\), its set of \emph{neighbors} is given by \(N_G(v) := \{u \in V_G \mid (v, u) \in E_G\}\). A \emph{homomorphism} from a graph \(G\) to another graph \(H\) is a mapping \(h : V_G \to V_H\) such that each edge \((v, w) \in E_G\) is mapped to an edge \(\myl(h(v), h(w)\myr) \in E_H\). An \emph{isomorphism}, on the other hand, is a bijective function \(f : V_G \to V_H\) that preserves adjacency: \((v, w) \in E_G\) if and only if \(\myl(f(v), f(w)\myr) \in E_H\). The notation \(\Hom(G, H)\) refers to the \emph{number of homomorphisms} from \(G\) to \(H\), and the function \(\Hom_G(\cdot)\) maps any graph \(H\) to \(\Hom(G, H)\). Given a sequence \(\Fb = (F_1, F_2, \ldots)\) of graphs, we define \(\Hom_\Fb(\cdot)\) as \(\myl(\Hom_{F_1}( \cdot), \Hom_{F_2}(\cdot), \ldots\myr)\), a tuple of homomorphism counts.
A \emph{graph invariant} is any function \(\xi\) on graphs that is unchanged under isomorphisms, i.e., \(\xi(G) = \xi(H)\) when \(G\) and \(H\) are isomorphic. For instance, \(\Hom_\Fb(\cdot)\) serves as a graph invariant for any graph sequence \(\Fb\). 
Moreover, we introduce the concept of \emph{rooted graphs}, where each graph \(G^r\) has a distinguished root vertex \(r \in V_G\). For two rooted graphs \(G^r\) and \(H^s\), a homomorphism must also map the root \(r\) of \(G\) to the root \(s\) of \(H\). The notation \(\Hom_{F^r}(\cdot)\) captures the number of homomorphisms $\Hom(F^r,G^v)$ from a rooted graph \(F^r\) to any rooted pair \((G, v)\), where \(v\) is treated as the root of \(G\). Similarly, \(\Hom_{\Fb^r}(\cdot)\) is defined, capturing important \emph{vertex invariants} for pairs $(G,v)$.

\paragraph{Graph neural networks and WL.}
We extend these notions to \emph{featured graphs} \(G = (V_G, E_G, \zeta_G)\), where each vertex is endowed with a feature vector \(\zeta_G : V_G \to \mathbb{R}^{d_0}\) of some fixed dimension \(d_0 \in \mathbb{N}\). We focus on Message-Passing Neural Networks (MPNNs)~\citep{GilmerSRVD17}, enhanced with homomorphism counts from $\Fb$~\citep{Barcelo2021-rs}. For a sequence of rooted graphs \(\Fb = (F_1^r, F_2^r, \ldots)\), the initial vertex representation for a vertex $v\in V_G$ in an \(\Fb\)-MPNN is:\footnote{We ignore vertex features when considering homomorphisms.}
\[
\phi_{\Fb}^{(0)}(G, v) := \left( \zeta_G(v), \Hom(F_1^r, G^v), \Hom(F_2^r, G^v), \ldots \right).
\]
At each iteration (\emph{layer}) \(0 \leq \ell \leq L\), this representation is updated as follows:
\[
\phi_{\Fb}^{(\ell+1)}(G, v) := \upd^{(\ell)}\Bigl( \phi_{\Fb}^{(\ell)}(G, v), \agg^{(\ell)}\bigl( \{\!\!\{ \phi_{\Fb}^{(\ell)}(G, u) \mid u \in N_G(v)\}\!\!\} \bigr) \Bigr),
\]
where \(\{\!\!\{\cdot\}\!\!\}\) denotes a multiset, and \(\upd^{(\ell)}\) and \(\agg^{(\ell)}\) are differentiable \emph{update} and \emph{aggregation} functions, respectively. After \(L\) iterations, a final pooling operation produces the graph-level representation:
$
\phi_{\Fb}^L(G) := \readout\bigl( \{\!\!\{ \phi_{\Fb}^{(L)}(G, v) \mid v \in V_G\}\!\!\} \bigr)$,
with \(\readout\) being a differentiable function. This construction defines a graph invariant.
We also consider the  \(\WLF{\Fb}\) algorithm, as introduced by~\citet{Barcelo2021-rs} as an extension of the one-dimensional Weisfeiler-Leman algorithm.
The  \(\WLF{\Fb}\) algorithm iteratively updates vertex colors. Initially, each vertex is assigned a color:
\[
\mathsf{wl}_{\Fb}^{(0)}(G, v) := \myl( \zeta_G(v), \Hom(F_1, G^v), \Hom(F_2, G^v), \ldots \myr).
\]
At each iteration \(0 \leq \ell \leq L\), new colors are assigned as follows:
\[
\mathsf{wl}_{\Fb}^{(\ell+1)}(G, v) := \bigl( \mathsf{wl}_{\Fb}^{(\ell)}(G, v), \{\!\!\{ \mathsf{wl}_{\Fb}^{(\ell)}(G, u) \mid u \in N_G(v)\}\!\!\} \bigr).
\]
The final graph invariant is \(\mathsf{wl}_{\Fb}^{(L)}(G) := \{\!\!\{ \mathsf{wl}_{\Fb}^{(L)}(G, v) \mid v \in V_G\}\!\!\}\). This invariant can be viewed as a \emph{color histogram} in \(\mathbb{N}^c\), where \(c\) is the number of distinct colors, assuming a canonical ordering on colors. When the list \(\Fb\) is empty we recover the $\WL$ algorithm \citep{weisfeiler1968reduction}.

\paragraph{Graph encoders.} \emph{Graph encoders} are mappings $\phi$ from the set $\Gc$ of graphs to some \emph{embedding space} $\Zc$, typically residing in $\Rb^k$, for $k\in\Nb$. The space $\Zc$ is assumed to be a \emph{metric space} for a metric $d_\Zc$. Examples of graph encoders are  $\Hom_\Fb$, $\Fb$-MPNNs and $\WLF{\Fb}$, for any sequence $\Fb$ of graphs and number $L\in\Nb$ of iterations. We will develop bounds for general graph encoders.

\paragraph{Wasserstein distance.}
Let \(\|\!\cdot\!\|\) denote the Euclidean norm in \(\mathbb{R}^d\), for some \(d \in \mathbb{N}\). Given two distributions \(\mu\) and \(\nu\) on \(\Rb^d\), the \emph{p-Wasserstein distance} between \(\mu\) and \(\nu\) is defined as:
\[
\Wc_p(\mu, \nu) := \inf_{\pi \in \Pi(\mu, \nu)} \left( \mathbb{E}_{(x, y) \sim \pi} \|x - y\|^p \right)^{1/p},
\]
where \(\Pi(\mu, \nu)\) denotes the set of all couplings of \(\mu\) and \(\nu\), i.e., distributions \(\pi\) on \(\mathbb{R}^d \times \mathbb{R}^d\) with \(\mu\) and \(\nu\) as marginals. In what follows, we restrict our attention to the 1-Wasserstein distance.

\section{Graph Encoders: Key Properties}\label{sec:graphenc}
Before presenting our generalization gap bounds, we first establish crucial properties of (classes of) graph encoders that play a significant role in our analysis. In particular, we revisit the relationship between classes of graph encoders in terms of their \emph{distinguishing power}, i.e., their ability to map distinct graphs in $\Gc$ to distinct embeddings in their embedding spaces. 

\begin{definition}
Let $\phi: \Gc \to \Zc_\phi$ and $\phi': \Gc \to \Zc_{\phi'}$ be two graph encoders. We say that $\phi$ \emph{bounds $\phi'$ in distinguishing power}, denoted by $\phi \sqsubseteq \phi'$, if for any two graphs $G$ and $H$ in $\Gc$,
\[
\phi'(G) \neq \phi'(H) \Rightarrow \phi(G) \neq \phi(H).
\]
In other words, $\phi'$ cannot distinguish more graphs than $\phi$. 
\end{definition}

Similarly, for classes $\Phi$ and $\Phi'$ of graph encoders, we say that $\Phi$ \emph{bounds $\Phi'$ in distinguishing power}, denoted by $\Phi \sqsubseteq \Phi'$, if no encoder in $\Phi'$ can distinguish more graphs than any of the encoders in $\Phi$. That is, for all $\phi' \in \Phi'$, there exists a $\phi \in \Phi$ such that $\phi \sqsubseteq \phi'$. If both $\Phi\sqsubseteq \Phi'$ and $\Phi'\sqsubseteq\Phi$ hold, then we write $\Phi\equiv\Phi'$ and say that both classes have the same distinguishing power.

From the seminal papers by \citet{morris2019weisfeiler} and \citet{xu2018powerful}, we know that 
$\text{MPNN}(L)\equiv \WL(L)$, where the argument $L$ refers to the number of layers/iterations. Similarly, $\Fb$-$\text{MPNN}(L)\equiv \WLF{\Fb}(L)$ \citep{Barcelo2021-rs}. It is also known that $\Hom_{\Tc}\sqsubseteq \text{MPNN}$ where $\Tc$ consists of all trees \citep{Dell2018-bg}, and $\Hom_{\Tc\circ\Fb}\sqsubseteq \Fb\text{-}\text{MPNN}(L)$ where $\Tc\circ\Fb$ consists of trees joined with copies of graphs in $\Fb$ \citep{Barcelo2021-rs}. Recent work by \citet{Neuen24} provides valuable insights
comparing $\Hom_{\Fb}$ for various $\Fb$ (see also \citep{lanzinger2024on}).

When graph encoders are comparable in terms of distinguishing power, one can recover the least expressive encoder from the most expressive one. This is formalized in the following lemma. Proofs in this section can be found in \cref{app:proofsecfour}.

\begin{restatable}{lemma}{factorization}
\label{lem:factor}
Let $\phi: \Gc \to \Zc_\phi$ and $\phi': \Gc \to \Zc_{\phi'}$ be two graph encoders such that $\phi \sqsubseteq \phi'$ holds. Then there exists a function $f: \Zc_\phi \to \Zc_{\phi'}$ such that $\phi' = f \circ \phi$.
\end{restatable}
As an illustration, consider an $L$-layer MPNN $\mathsf{M}$; we know that $\WL(L)\sqsubseteq \mathsf{M}$. It now suffices to define $f$ such that it maps a color histogram $\mathbf{h}$ to $\mathsf{M}(G)$, the embedding of $G$ by $\mathsf{M}$, where $G$ is a graph satisfying $\mathsf{wl}^{(L)}(G) = \mathbf{h}$. This is well-defined due to the earlier observation that $\WL(L) \sqsubseteq \mathsf{M}$.

For some classes of graph encoders, the function $f$ satisfies additional desirable properties, as we explain next. 
We say that a graph encoder $\phi: \Gc \to \Zc_\phi$ is \emph{$B$-bounded} if $d_{\Zc_\phi}(\phi(G), \phi(H)) \leq B$ for any $G, H \in \Gc$. Furthermore, a graph encoder $\phi: \Gc \to \Zc_\phi$ is \emph{$S$-separating} if $d_{\Zc_\phi}(\phi(G), \phi(H)) \geq S$ for any $G, H \in \Gc$ such that $\phi(G) \neq \phi(H)$. We recall that a function $f$ between metric spaces $\Zc$ and $\Zc'$ is \emph{Lipschitz} with constant $\operatorname{Lip}(f)$ if for any $z_1,z_2\in\Zc$, $d_{\Zc'}(f(z_1),f(z_2))\leq \operatorname{Lip}(f)d_\Zc(z_1,z_2)$. For simplicity, we
set $\operatorname{Lip}(f)=\infty$ when $f$ is not Lipschitz for a finite constant.

\begin{restatable}{proposition}{SB}\label{prop:specificL}
Let $\phi: \Gc \to \Zc_\phi$ be an $S$-separating graph encoder and $\phi': \Gc \to \Zc_{\phi'}$ be a $B$-bounded graph encoder such that $\phi \sqsubseteq \phi'$. Then $\phi=f\circ\phi'$ for a function $f: \Zc_\phi \to \Zc_{\phi'}$ which is Lipschitz with constant $\operatorname{Lip}(f) = B/S$.
\end{restatable}

There are plenty of bounded graph encoders; indeed, just consider any GNN employing bounded-range activation functions such as 
sigmoid, tanh, truncated ReLU \citep{hamilton2017inductive}. Other examples include normalized homomorphism count vectors or color histograms \citep{lovasz2006limits}. Similarly, any graph encoder mapping graphs into a discrete subset of $\Rb^d$ is $S$-separating. For example, any $\Hom_{\Fb}$ is $1$-separating since whenever $\Hom_\Fb(G) \neq \Hom_\Fb(H)$, there exists an $F \in \Fb$ such that $\Hom(F, G) \neq \Hom(F, H)$. Since the latter are natural numbers, and assuming a discrete metric $d$, $d\myl(\Hom(F, G), \Hom(F, H)\myr) \geq 1$. A similar argument applies to graph encoders based on $\WL$ or its higher-order variant $\WLF{k}$.

Our generalization bounds use the 1-Wasserstein distance between distributions, as we will see shortly. Using Proposition~\ref{prop:specificL}, and in particular the Lipschitz property, we can relate the Wasserstein distance between the pushforward distributions of distributions $\mu$ and $\nu$ on $\Gc$ for the embedding spaces of the graph encoders. Formally, let $\phi:\Gc\to\Zc_\phi$ be a graph encoder and let $\mu$ be a distribution on $\Gc$. Then the \emph{pushforward distribution of \(\mu\) under $\phi$} is the distribution on $\Zc_\phi$ given by 
\[
\phi_\sharp(\mu)(z):= \mu\bigl(\{G \in \Gc \mid \phi(G) = z\}\bigr),
\]
where \(z\) is an element in the embedding space \(\mathcal{Z}_{\phi}\). We can now state the proposition.

\begin{restatable}{proposition}{Lipschitzfactor}\label{prop:wasserineq}
Let $\phi: \Gc \to \Zc_\phi$ and $\phi': \Gc \to \Zc_{\phi'}$ be two graph encoders such that $\phi' = f \circ \phi$.
Then for any distributions $\nu$ and $\nu'$ over $\Gc$, we have that the inequality
$
\mathcal{W}_1\myl(\phi_\sharp'(\nu), \phi_\sharp'(\nu')\myr) \leq \operatorname{Lip}(f) \cdot \mathcal{W}_1\myl(\phi_\sharp(\nu), \phi_\sharp(\nu')\myr)$ holds.
\end{restatable}

We remark that the inequality above becomes vacuous when $f$ is not Lipschitz and hence $\operatorname{Lip}(f)=\infty$. As an important corollary of Propositions~\ref{prop:specificL} and~\ref{prop:wasserineq}, we obtain the following.

\begin{corollary}\label{cor:usehomtobound}
Let $\phi: \Gc \to \Zc_\phi$ be an $S$-separating graph encoder and $\phi': \Gc \to \Zc_{\phi'}$ a $B$-bounded graph encoder such that $\phi \sqsubseteq \phi'$ holds. Then for any  distributions $\nu$ and $\nu'$ over $\Gc$, we have 
\[
\mathcal{W}_1\myl(\phi_\sharp'(\nu), \phi_\sharp'(\nu')\myr) \leq (B/S) \cdot \mathcal{W}_1\myl(\phi_\sharp(\nu), \phi_\sharp(\nu')\myr).
\]
\end{corollary}

As an example, consider a $B$-bounded graph encoder $\phi: \Gc \to \Zc_\phi$ which is bounded in distinguishing power by the $1$-separating encoder $\mathsf{Hom}_\Fb$, for some sequence $\Fb$ of graphs. Then for any distributions $\nu$ and $\nu'$ on $\Gc$, we have
\[
\mathcal{W}_1\myl(\phi_\sharp(\nu), \phi_\sharp(\nu')\myr) \leq (B/1) \cdot \mathcal{W}_1\myl((\mathsf{Hom}_\Fc)_\sharp(\nu),( \mathsf{Hom}_\Fc)_\sharp(\nu')\myr).
\]
More broadly, these results suggest that the variance of embedding distributions in $\Zc_\phi$, produced by a complex graph encoder, can be effectively upper bounded by the variance of simpler, combinatorial graph invariants—such as homomorphism counts, Weisfeiler-Leman tests, and other structural descriptors, provided that the latter bound the former in terms of distinguishing power.

\section{Generalization Analysis}\label{sec:genal}
We use the setup from \citet{Chuang2021-ik} but 
 translated to the graph setting.
More precisely, let \(\Gc\) represent the input space of graphs, \(\Zc\) the embedding space in $\Rb^d$ for some $d\in\Nb$, and \(\Yc = \{1, \dots, K\}\) the output space consisting of \(K\) classes. We define a set of \emph{graph encoders} \(\Phi = \{\phi: \Gc \to \Zc\}\) and a set of \emph{predictors} \(\Psi = \{\psi=(\psi_1,\ldots,\psi_K): \Zc \to \Rb^K\}\). A \emph{score-based graph classifier} $\psi\circ\phi$ simply returns \(\arg\max_{y\in\Yc} \psi_y(\phi(G))\) on input $G$. The graph encoders in $\Phi$ are assumed to be graph invariants, such as, e.g., $\Fb$-$\text{MPNNs}$, $\WLF{\Fb}$, or $\Hom_\Fb$. 

We define the \emph{margin} of a graph classifier \(\psi \circ \phi\) for a graph sample \((G, y) \in \Gc \times \Yc\) as
\begin{equation*}
\label{eqn: margin_loss}
    \rho_\psi(\phi(G), y) := \psi_y(\phi(G)) - \max _{y' \neq y} \psi_{y'}(\phi(G)).
\end{equation*}
The graph classifier \(\psi \circ \phi\) misclassifies \(G\) if \(\rho_\psi(\phi(G), y) < 0\). Let \(\mu\) be a distribution over \(\Gc \times \Yc\), and  \(\Sc = \{(G_i, y_i)\}_{i=1}^m\) be a set of \(m\) graph samples drawn i.i.d. from \(\mu\), i.e., \(\Sc \sim \mu^m\). The \emph{empirical distribution} \(\mu_\Sc\) is defined as $
\mu_\Sc := \frac{1}{m}\sum_{i=1}^m \delta_{(G_i, y_i)}$, where \(\delta_{(G_i, y_i)}\) denotes the Dirac delta measure centered at \((G_i, y_i)\). 
The \emph{expected zero-one loss} $R_\mu(\psi \circ \phi)$ and the \emph{$\gamma$-margin empirical zero-one loss} $\hat{R}_{\gamma, \Sc}(\psi \circ \phi)$ are defined as \begin{equation*}
    R_\mu(\psi \circ \phi):=\mathbb{E}_{(G, y) \sim \mu}\left[\mathbbm{1}_{\rho_\psi(\phi(G), y) \leq 0}\right] \text{ and }\hat{R}_{\gamma, \Sc}(\psi \circ \phi):=\mathbb{E}_{(G, y) \sim \mu_\Sc}\left[\mathbbm{1}_{\rho_\psi(\phi(G), y) \leq \gamma}\right].
\end{equation*} 
We aim to bound the \emph{generalisation gap} $R_\mu(\psi \circ \phi) - \hat{R}_{\gamma, \Sc}(\psi \circ \phi)$ for the graph classifier $\psi \circ \phi$. 

\subsection{Generalization bound}
We are now ready to present the generalization bounds. Our results build on the margin bounds of \citet{Chuang2021-ik}, which are themselves based on a generalized notion of variance that involves the Wasserstein distance \citep{solomonGN22}. This notion more effectively captures the structural properties of the feature distribution. Crucially, we fully exploit the properties of graph encoders and, in particular, use \cref{prop:wasserineq} to derive an upper bound on the generalization gap of any graph encoder $\phi: \Gc \to \Zc_\phi$
in terms of \emph{any} graph encoder bounding $\phi$ in distinguishing power! 
    
In order to formally state our results, some additional definitions are needed. Recall that we consider graph classifiers $(\psi_1,\ldots,\psi_K)\circ\phi$ where $\phi:\Gc\to\Zc_\phi$ is a graph encoder and the predictor $\psi=(\psi_1,\ldots,\psi_K)$ is such that each $\psi_i:\Zc_\phi\to\Rb$. Recall also that the output space $\Yc=\{1,\ldots,K\}$ and that $\mu$ is a distribution over $\Gc\times\Yc$. We denote by $\mu_x$ the marginal distribution on $\Gc$, i.e., $\mu_x(G):=\int \mu(G,y) \text{d}y$ and by $\mu_y$ the  marginal distribution on $\Yc$, i.e, $\mu_y(c):=\int \mu(x,c) \text{d}x$. Then,
for each $c\in \Yc$, $\mu_c(G)$ is the conditional distribution on $\Gc$ defined by $\mu(G,c)/\mu_y(c)$.

\begin{restatable}{theorem}{onewexpectationbound}
\label{lemma:1w_expectation_bound}
Fix \(\gamma > 0\) and a graph encoder \(\phi:\Gc\to\Zc_\phi\). Let $\lambda:\Gc\to\Zc_\lambda$ be a graph encoder that bounds $\phi$ in distinguishing power, i.e., $\lambda\sqsubseteq\phi$. Then, for every distribution $\mu$ on $\Gc\times\Yc$, for every predictor $\psi=(\psi_y)_{i\in\Yc}$ and every $\delta\in(0,1)$, with probability at least $1-\delta$ over all choices of $\Sc\sim \mu^m$, we have that the generalization gap $R_\mu(\psi \circ \phi) - \hat{R}_{\gamma, \Sc}(\psi \circ \phi)$ is upper bounded by
\begin{equation}  \label{eq:main}  
\mathbb{E}_{c \sim \mu_y} \left[ \frac{\operatorname{Lip}\left(\rho_\psi(\cdot, c)\right)\operatorname{Lip}(f)}{\gamma} 
        \mathbb{E}_{T,\tilde{T} \sim \mu_{c}^{m_c}} \left[ 
        \mathcal{W}_1\bigl(\lambda_\sharp(\mu_{c,T}), \lambda_\sharp(\mu_{c,\tilde{T}})\bigr) 
         \right]
    \right]  + \sqrt{\frac{\log (1 / \delta)}{2m}}, \tag{$\dagger$}
\end{equation}
where $\phi=f\circ\lambda$ and for each $c\in\Yc$, $m_c$ denotes the number of pairs $(G,c)$ in $\Sc$.  Also, recall that for $T\sim \mu_c^{m_c}$,
$\mu_{c,T}$ is the empirical distribution $\mu_{c,T}:=\sum_{G\in T} \delta_{G}$; similarly for $\mu_{c,\tilde T}$.
\end{restatable}

The proof is a consequence of Theorem 2 in \citet{Chuang2021-ik} and \cref{prop:wasserineq}. As also observed by those authors, the expectation term over \(T,\tilde{T} \sim \mu_c^{m_c}\) is intractable in general. To address this drawback, \citet{Chuang2021-ik} show how to estimate the expectation by means of \emph{sampling}, provided that encoders are $B$-bounded. A similar approach works in our case as well. We refer to \cref{app:secfive} for proofs and details.

 \cref{lemma:1w_expectation_bound}  highlights several key factors that influence the generalization of graph classifiers: (i) the learning behavior of the predictors $\psi$, captured by \(\operatorname{Lip}\left(\rho_\psi(\cdot, c)\right)\); (ii) the learning behavior of graph encoder $\phi$, relative to $\lambda$, described by \(\operatorname{Lip}(f)\); and (iii)  the variance of graph structures, in the Wasserstein distance \(\mathcal{W}_1\bigl(\lambda_\sharp(\mu_{c,T}), \lambda_\sharp(\mu_{c,\tilde{T}})\bigr)\) for graph samples $T,\tilde T\sim\mu_c^{m_c}$.

\subsection{Concentration and separation}
In terms of concentration, since \(\mathbb{E}_{T,\tilde{T} \sim \mu_c^{m_c}}[W_1(\fa_\sharp(\mu_{c,T}), \fa_\sharp(\mu_{c,\tilde{T}}))] \leq O(m^{-1/d})\) \citep{Chuang2021-ik}, a large sample size \(m\) and a small dimension \(d\) of the embedding space \(\Zc_\lambda\) lead to a smaller generalization bound. For instance, when \(\mu\) is concentrated on graphs in \(\Gc\) with low \emph{color complexity} \citep{morris23meet}—i.e., the \(\WL\) test requires only a small number of colors for the graph's vertices—combinatorial graph encoders like \(\Hom_\Fb\) and \(\WLF{\Fb}(L)\) can operate in low-dimensional spaces. This observation is consistent with earlier findings \citep{KieferM20, garg2020radmacher, liao21pacbayes, ju2023diffusion, CongRM21, esser2021sometimes, morris23meet} about the effect of graph size, degree, and maximum degree on generalization performance.

Of particular interest is the case when the bounding graph classifier $\lambda$ is assumed to have a large margin. A larger margin is generally associated with better generalization~\citep{elsayed2018large,Chuang2021-ik}. If we assume the margin \(\gamma\) is satisfied for $\psi\circ\lambda$, for all graph samples, and for each $c\in\Yc$, the predictor \(\psi_c \in \psi\) is Lipschitz, then (see Lemma 10 in \citet{Chuang2021-ik}) we have 
\[
\gamma\leq \bigl(\max_{\substack{c,c'\in\Yc\\c\neq c'}}W_1(\lambda_\sharp(\mu_c), \lambda_\sharp(\mu_{c'}))\bigr)\bigl(\max_{c\in\Yc}\Lip(\psi_c)\bigr).
\]
which allows us to lower bound $1/\gamma$ in \cref{eq:main}, resulting in \cref{pro:lowerbound_lip}, hereby revealing a trade-off between concentration and separation.

\begin{restatable}{proposition}{lowerboundlip}
\label{pro:lowerbound_lip}
Under the same assumptions as in \cref{lemma:1w_expectation_bound}, but with the additional requirement that the predictors $\psi_c$ in $\psi$ are Lipschitz, and that the bounding graph classifier $\lambda$ has a large margin, i.e., \(\rho_{\psi}(\fa(G), y) \geq \gamma\) for all $(G,y)\sim\mu$, then for any $\delta\in(0,1)$, with probability at least $1-\delta$ over all choices $\Sc\sim\mu^m$, we have that the gap $R_\mu(\psi \circ \phi) - \hat{R}_{\gamma, \Sc}(\psi \circ \phi)$ is upper bounded by

\begin{equation*}
\frac{\Lip(f)\cdot \mathbb{E}_{c\sim \mu_y}
    \left[\Lip(\rho_\psi(\cdot, c))
         \mathbb{E}_{T,\tilde{T}\sim \mu_{c}^{m_c}}
      \left[ 
      \W_1\left(\fa_\sharp(\mu_{c,T}), \fa_\sharp(\mu_{c,\tilde{T}})\right) 
      \right]
    \right]}{\left(\max_{c\in\Yc} \Lip(\psi_c)\right)\left(\max_{c,c'\in\Yc, c\neq c'} W_1\left(\fa_\sharp(\mu_c), \fa_\sharp(\mu_{c'})\right)\right)}
     + \sqrt{\frac{\log (1 / \delta)}{2 m}}.
\end{equation*}

Furthermore, this bound is no worse than the one given in \cref{lemma:1w_expectation_bound}.
\end{restatable}

The above proposition highlights that, to achieve a low generalization bound, it is crucial to ensure good concentration between embeddings of the same class, i.e., \(W_1(\fa_\sharp(\mu_{c,T}), \fa_\sharp(\mu_{c,\tilde{T}}))\), while maintaining a large separation between embeddings of different classes, i.e., \(W_1(\fa_\sharp(\mu_c), \fa_\sharp(\mu_{c'}))\), where \(c \neq c'\). This can be achieved when $\fa$ learn embeddings in the ``right'' directions, where embeddings of different classes are ``more separated'' than those of the same class, or when the distribution $\mu$ is concentrated on graphs for which this separation happens for $\lambda$.

\paragraph{Remarks.}
In our bounds, we identify three Lipschitz constants: $\Lip(\rho_\psi(\cdot, c))$, $\Lip(\psi_c)$, and $\Lip(f)$.
First, note that $\rho_\psi(\cdot, c)$ depends on $\psi_c$, and therefore it is Lipschitz in its first argument if $\psi_c$ is Lipschitz. For simplicity, we assume that the predictor $\psi = (\psi_c)_{c \in \Yc}$ is a softmax function with Lipschitz constant $1$.
For general $\psi_c$, $\Lip(\rho_\psi(\cdot, c))$ can be approximated empirically using the Jacobian, as suggested by \citep{Chuang2021-ik}.

 Furthermore, \cref{cor:usehomtobound}  states that the connecting function $f$ between the graph encoders $\lambda \sqsubseteq \phi$ is Lipschitz with constant $B/S$, provided that $\phi$ is $B$-bounded and $\lambda$ is separating. Therefore, when $\phi$ is $B$-bounded, $\Lip(f)$ decreases as $S$ increases. We also note that $S$ can increase with added expressivity in $\lambda$, which enhances its separation ability. In practice, both $B$ and $S$ can be computed empirically. We discuss the effect of added expressivity in $\lambda$ in more detail in the next section. 

  \begin{table}[t!]
     \begin{center}
      \resizebox{\textwidth}{!}{
     \begin{tabular}{c c |c  c c}
     \toprule
       \multicolumn{2} {c|}{\hspace{1cm}Initial Vertex Colors}  & After One Iteration\hspace{0.5cm} & Graph Embeddings&Wasserstein\\
       & \hspace{-0.4cm}{\small $G$\hspace{1.5cm}$G'$}&\hspace{-0.4cm}{\small $G$\hspace{1.5cm}$G'$}& (Difference)&Distance\\
       \midrule
        \shortstack{(a)\\$\WL$\\\vspace{0.8cm}} &\includegraphics[width=0.25\textwidth]{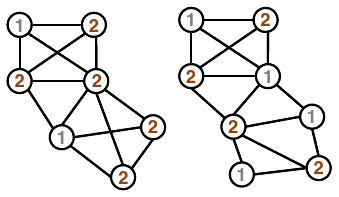}
        & 
        \includegraphics[width=0.25\textwidth]{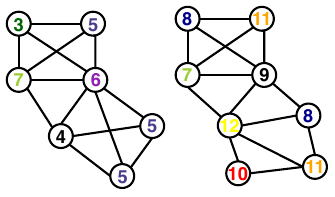}
        & 
        \shortstack{\includegraphics[width=2.8cm]{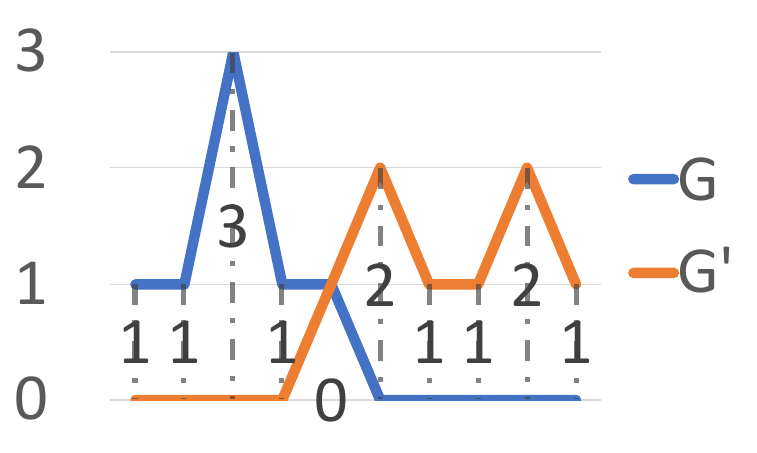}
 }
 


        &\shortstack{4.796\\\vspace{0.8cm}}\\ 
        \hline
        \shortstack{(b)\\$\WLF{C_4}$\\\vspace{0.8cm}}  &\includegraphics[width=0.25\textwidth]{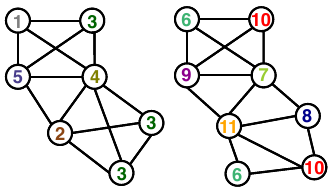}
        & 
        \includegraphics[width=0.25\textwidth]{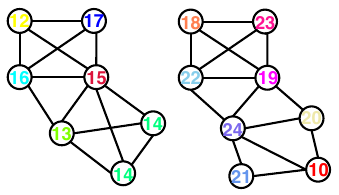}
        & 
        \shortstack{\includegraphics[width=2.8cm]{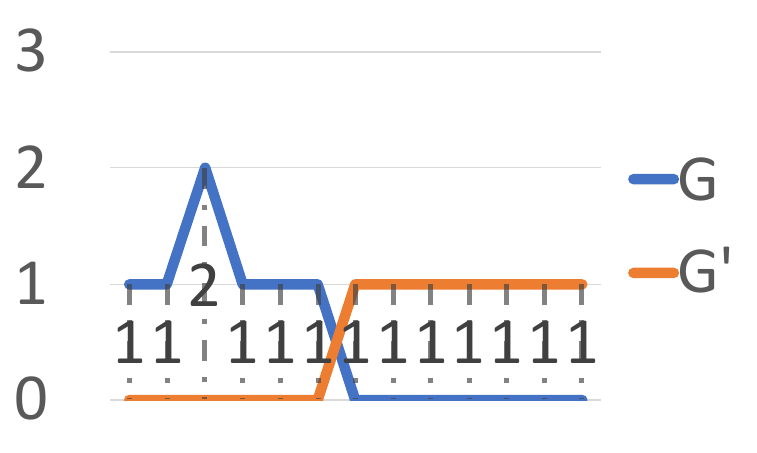}
        }

 
 & \shortstack{4.123\\\vspace{0.8cm}}\\  \hline
\shortstack{(c)\\$\WLF{K_4}$\\\vspace{0.8cm}}  &\includegraphics[width=0.25\textwidth]{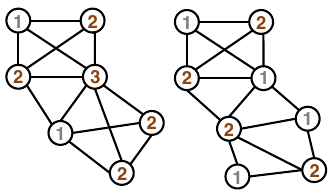}
      & 
\includegraphics[width=0.25\textwidth]{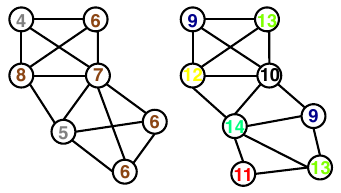}
      & 
 \shortstack{\includegraphics[width=2.8cm]{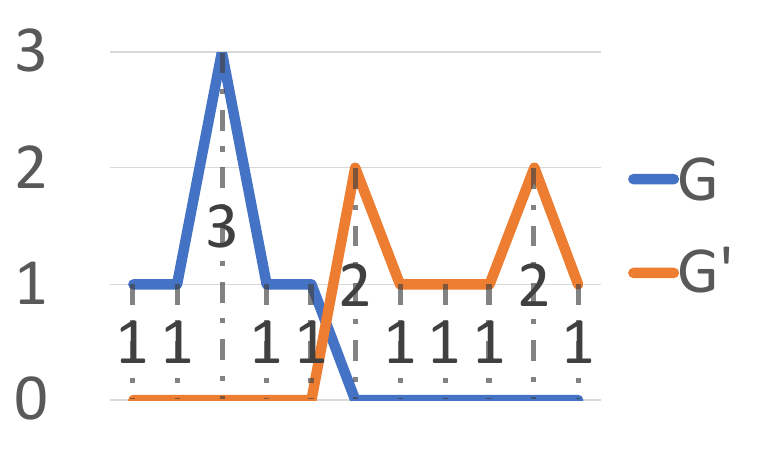}
 }
     & \shortstack{5.000\\\vspace{0.8cm}}\\ 
      \bottomrule
      \end{tabular}
      }
\caption{Two graphs with the initial vertex colors, including the input vertex features and homomorphism counts, the vertex colors after one iteration, and the dimension-wise difference and Wasserstein distance between graph embeddings for three models: (a) $\WL$, (b) $\WLF{C_4}$, and (c) $\WLF{K_4}$. } 
\label{fig:hom_emb_C4}
      \label{tbl:case-studies}
      \end{center}
      \end{table}

\section{Case Studies}
In this section, we present case studies to illustrate how model expressivity influences the generalization ability of graph encoders. Recall that two key factors are considered in \cref{lemma:1w_expectation_bound} and \cref{pro:lowerbound_lip}: (i) \emph{intra-class concentration}, which reflects the variance of graph structures, measured by the Wasserstein distance \(\mathcal{W}_1(\lambda_\sharp(\mu_{c,T}), \lambda_\sharp(\mu_{c,\tilde{T}}))\) for graph samples \(T, \tilde{T} \sim \mu_c^{m_c}\), and (ii) \emph{inter-class separation}, measured by the Wasserstein distance \(\max_{c, c' \in \Yc, c \neq c'} \mathcal{W}_1(\fa_\sharp(\mu_c), \fa_\sharp(\mu_{c'}))\). For simplicity, we consider the following graphs \(\{G, G', H, H'\}\) from the PROTEINS dataset~\citep{Morris+2020}, uniformly selected by the distribution $\mu$,
\[
G=\graphGM\hspace{1cm}G'=\graphGGM \hspace{1cm} H=\graphHM\hspace{1cm} H'=\graphhM
\]
Here, \(G\) and \(G'\) belong to the class \(c\), while \(H\) and \(H'\) belong to the class \(c'\), where \(c \neq c'\). Assuming the margin condition is satisfied for all classes, including \(c\) and \(c'\), and that embeddings of graphs within each class cluster around the embeddings of \(G\), \(G'\), and \(H\), \(H'\), we estimate \(\mathcal{W}_1\myl(\lambda_\sharp(\mu_{c,T}), \lambda_\sharp(\mu_{c,\tilde{T}})\myr)\) using \(\mathcal{W}_1\myl(\lambda_\sharp(G), \lambda_\sharp(G')\myr)\). Since there are only two classes in this dataset, we estimate \(\max_{c, c' \in \Yc, c \neq c'} \mathcal{W}_1\myl(\fa_\sharp(\mu_c), \fa_\sharp(\mu_{c'})\myr)\) by \(\mathcal{W}_1\myl(\fa_\sharp(\mu_c), \fa_\sharp(\mu_{c'})\myr)\) where $\mu_c=\{G,G'\}$ and $\mu_{c'}=\{H,H'\}$.

For the bounding graph encoders \(\lambda\), we use $\WL$ as the base model. Then, following the approach by \cite{Barcelo2021-rs} 
we consider two simple rooted graphs $\graphC$ and $\graphK$ and append their homomorphism counts $\Hom(\graphC,\cdot)$ and $\Hom(\graphK,\cdot)$ to the vertex feature of each rooted pair \((\cdot, v)\) in the graphs $G$, $G'$, $H$ and $H'$, respectively. This leads to slight increase in model expressivity, compared with $\WL$, allowing us to analyze how these differences impact key factors in generalization. We refer to these two graph encoders as $\WLF{C_4}$, where homomorphism counts of $\graphC$ are added, and $\WLF{K_4}$, where homomorphism counts of $\graphK$ are added, respectively.

\cref{tbl:case-studies} presents the graphs \(G\) and \(G'\) along with their initial vertex colors, including input vertex features and homomorphism counts, as well as the vertex colors after one iteration. It also includes the difference between the graph embeddings for the three models. 
As model expressivity increases, we observe two distinct scenarios:
\begin{itemize}
    \item[(1)] \emph{More expressivity leads to better generalization}: 
Compared to the graph embeddings from $\WL$, incorporating $C_4$ improves intra-class concentration. The homomorphism counts of $C_4$ reduce the variance between graph embeddings as shown in \cref{tbl:case-studies}, resulting in a distance of 4.123, smaller than 4.796 for $\WL$.
     \item[(2)] \emph{More expressivity leads to worse generalization}: When $K_4$ is used, the graph embeddings of $G$ and $G'$ yield a distance of 5.000, which is larger than its $\WL$ counterpart 4.796. Compared to $\WLF{C_4}$, each dimension of the $\WLF{K_4}$ embeddings has the same or larger magnitude,  reflecting higher variance in the graph embeddings.
\end{itemize}
When measuring inter-class separation using \(\mathcal{W}_1(\fa_\sharp(\mu_c), \fa_\sharp(\mu_{c'}))\), the models $\WL$, $\WLF{C_4}$, and $\WLF{K_4}$ achieve distances of 4.582, 4.511, and 4.840, respectively. These results suggest a narrowing in the gaps of these models, compared to intra-class concentration alone. The trends in inter-class separation may change depending on the graph structure. 
For instance, if graphs of class \(c'\) cluster around the embedding of \(H'\), i.e., estimating \(\mathcal{W}_1\myl(\fa_\sharp(\mu_c), \fa_\sharp(\mu_{c'})\myr)\) with $\mu_c=\{G,G'\}$ and $\mu_{c'}=\{H'\}$, the reverse trend may occur, with \(\WL\) achieving a distance of 4.796 and \(\WLF{K_4}\) achieving 4.583. This highlights the importance of inter-class separation in balancing a model's generalization performance alongside intra-class concentration.

\section{Experiments}
\label{sec:empirical}

\textbf{Tasks and Datasets $\quad$}
We conduct graph classification experiments on six widely used benchmark datasets: ENZYMES, PROTEINS, and MUTAG from the TU dataset collection~\citep{Morris+2020}, as well as SIDER and BACE from the molecular dataset collection~\citep{WuRFGGPLP17}. For SIDER, which comprises 27 classification tasks, we focus specifically on the 21st task. Each dataset is randomly divided into training and test sets following a 90\%/10\% split.

\textbf{Setup and Configuration $\quad$}
Each classification task is trained for 400 epochs, with five independent runs to report the mean and standard deviation of the results. Consistent with the setup in~\citet{tang2023towards,morris23meet,CongRM21}, we eliminate the use of regularization techniques such as dropout and weight decay. A batch size of 128 is utilized, with a learning rate set to $10^{-3}$, and the hidden layer dimension fixed at 64. The margin loss function is employed with a margin parameter $\gamma = 1$. To compute the generalization gap, we utilize the sample-based variant of the bound as outlined in \cref{lemma:1w_expectation_bound}, as given in \cref{lemma:1w_sample_bound} of the appendix.

For the graph encoder $\phi$, we adopt both MPNNs and $\Fc$-MPNNs, with expressivity constraints defined by $\WL$ and $\WLF{\Fb}$, respectively, as described in \cref{sec:graphenc}. The predictor $\psi(\cdot)$ is modeled using the softmax function, which has a Lipschitz constant of 1~\citep{gao_softmax}, ensuring that $\Lip(\rho_{\psi}(\cdot, c))$ is also 1. 
We estimate $\Lip(f)$ as: $
\Lip(f) = \max_{G,H \in \Gc_{\text{train}}} \left( \frac{d_{\Zc_\phi}(\phi(G), \phi(H))}{d_{\Zc_\fa}(\fa(G), \fa(H))} \right)$,
where $G$ and $H$ are sampled from the training set $\Gc_{\text{train}}$. For all experiments, we set the confidence level $\delta$ to 0.1, yielding bounds with high probability. 
The experiments are conducted on a Linux-based system equipped with 96 GB of memory and a single NVIDIA RTX A6000 GPU.

\begin{table}[h]
    \caption{Graph classification gaps with different numbers of MPNN layers. The MPNN embeddings are not normalized.}
    \label{tab:graph_layer_gap}
    \centering
    \resizebox{0.8\textwidth}{!}{
    \begin{tabular}{clrrrrr}
    \toprule
    \multirow{2}{*}{} &  & \multicolumn{4}{c}{Dataset} \\ \cmidrule{3-7}
     \# Layers&  & ENZYMES & PROTEINS & MUTAG & SIDER & BACE \\
     \midrule
     \multirow{4}{*}{1} & Acc. gap & 25.41\sd{3.82}&2.53\sd{1.76}&-4.84\sd{2.52}&4.21\sd{0.71}&3.63\sd{1.63} \\
     & Loss gap & 0.248\sd{0.040}&0.029\sd{0.015}&-0.070\sd{0.017}&0.037\sd{0.003}&0.018\sd{0.017} \\
     & Our Bound & 7.926\sd{1.279}&2.193\sd{0.702}&1.216\sd{0.169}&0.511\sd{0.286}&1.479\sd{0.301} \\
     & VC Bound & 586 & 929 & 51 & 960 & 621\\ \midrule
     \multirow{4}{*}{2} & Acc. gap & 24.26\sd{2.92}&3.62\sd{1.06}&-4.84\sd{1.78}&4.77\sd{1.12}&4.33\sd{1.29} \\
     & Loss gap & 0.242\sd{0.026}&0.032\sd{0.010}&-0.074\sd{0.007}&0.038\sd{0.003}&0.037\sd{0.019} \\
     & Our Bound & 7.425\sd{0.982}&1.404\sd{0.144}&1.247\sd{0.155}&0.620\sd{0.463}&1.729\sd{0.251} \\
     & VC Bound & 595 & 996 & 121 & 1300 & 1060\\ \midrule
     \multirow{4}{*}{3} & Acc. gap & 29.89\sd{3.01}&2.95\sd{1.47}&-6.60\sd{2.57}&4.03\sd{0.19}&4.99\sd{1.22} \\
     & Loss gap & 0.237\sd{0.035}&0.025\sd{0.009}&-0.058\sd{0.012}&0.038\sd{0.002}&0.032\sd{0.011} \\
     & Our Bound & 6.513\sd{0.951}&1.421\sd{0.220}&1.649\sd{0.158}&0.409\sd{0.253}&1.789\sd{0.226}\\
     & VC Bound & 595 & 996 & 135 & 1309  & 1089\\ \midrule
     \multirow{4}{*}{4} & Acc. gap & 27.04\sd{4.69}&2.86\sd{0.86}&-6.71\sd{2.00}&4.13\sd{0.03}&3.13\sd{2.04} \\
     & Loss gap & 0.235\sd{0.038}&0.027\sd{0.005}&-0.073\sd{0.009}&0.036\sd{0.001}&0.022\sd{0.030} \\
     & Our Bound & 6.825\sd{0.796}&1.434\sd{0.297}&1.535\sd{0.115}&0.298\sd{0.080}&1.686\sd{0.377} \\
     & VC Bound & 595 & 996 & 139 & 1309 & 1093\\ \midrule
     \multirow{4}{*}{5} & Acc. gap & 29.85\sd{4.95}&0.79\sd{0.78}&-4.01\sd{1.14}&4.09\sd{0.06}&3.53\sd{1.41} \\
     & Loss gap & 0.256\sd{0.037}&0.020\sd{0.007}&-0.071\sd{0.021}&0.035\sd{0.001}&0.020\sd{0.020} \\
     & Our Bound & 6.384\sd{0.813}&1.308\sd{0.165}&1.773\sd{0.194}&0.369\sd{0.172}&1.662\sd{0.120} \\
     & VC Bound & 595 & 996 & 139 & 1309 & 1093\\ \midrule
     \multirow{4}{*}{6} & Acc. gap & 28.85\sd{6.01}&1.42\sd{1.40}&-4.02\sd{3.68}&3.60\sd{0.48}&2.44\sd{1.40} \\
     & Loss gap & 0.264\sd{0.025}&0.030\sd{0.008}&-0.078\sd{0.019}&0.034\sd{0.002}&0.022\sd{0.016} \\
     & Our Bound & 6.151\sd{0.798}&1.340\sd{0.316}&1.627\sd{0.038}&0.353\sd{0.156}&1.785\sd{0.237} \\
     & VC Bound & 595 & 996 & 139 & 1309 & 1093\\ 
     \bottomrule
    \end{tabular}%
    }
\vspace{-1em}
\end{table}

\subsection{Results and discussion}

\textbf{How well can the bounds predict the generalization ability of MPNNs?} To answer this, we compare the proposed bound with empirical generalization gaps, measured by loss and accuracy, while controlling MPNN expressivity by varying the number of layers.
\cref{tab:graph_layer_gap} presents the proposed bound and the empirical generalization gaps for different numbers of MPNN layers across five datasets. For comparison, we also include the VC bound from \cite{morris23meet}, which is based on the number of unique color histograms produced by $\WL$, indicating the number of graphs distinguishable by $\WL$.
Our results show that the proposed bounds closely mirror the empirical generalization gaps across datasets and varying layer depths, effectively predicting generalization errors. This consistency highlights the bound’s ability to reflect changes in generalization performance as model depth increases. 
In contrast, the VC bound stabilizes after three iterations for most datasets, as $\WL$ can no longer distinguish more graphs. As a result, the VC bound fails to capture improved generalization performance at deeper layers and remains higher than our bound.
Our bound is less vacuous compared to the VC bound and other bounds, such as those proposed by \cite{garg2020radmacher, liao21pacbayes}, which tend to be on the order of $10^4$.  

To evaluate how well the proposed bound predicts the generalization gap of $\Fc$-MPNNs across different homomorphism pattern selections, we present the empirical loss gap and generalization bound for three distinct pattern sets, alongside MPNN, as shown in \cref{fig:patterns_gap}. We designate $P_n$, $K_n$, and $C_n$ as $n$-path, $n$-clique, and $n$-cycle graphs, respectively, and refer to the MPNN without any specific pattern as ``no pattern". It can be seen that the generalization bound closely aligns with the empirical gap across different pattern choices, with some exceptions in ENZYMES. Notably, the choice of pattern influences the generalization gap in different ways. In ENZYMES, cycle patterns lead to a larger gap compared to cliques and paths. In PROTEINS, using paths or cliques increases the generalization gap, while cycles reduce it. These changes in the empirical generalization gap are largely captured by the corresponding bounds.

\begin{figure}[t!]
\centering
\subfigure[PROTEINS, 4 layers]{\includegraphics[width=0.245\textwidth]{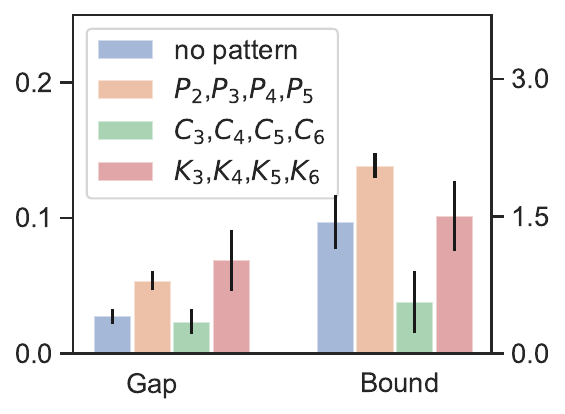}\label{fig:protein-4}}
\subfigure[PROTEINS, 6 layers]{\includegraphics[width=0.245\textwidth]{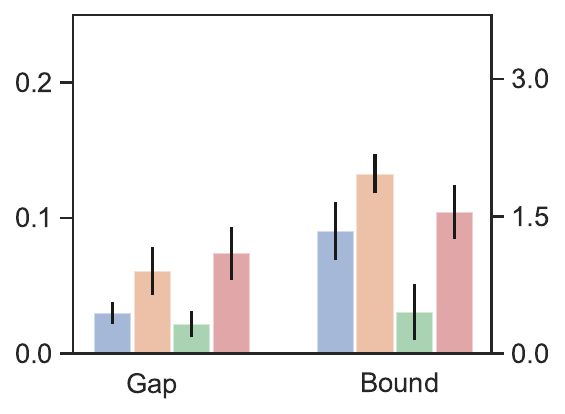}\label{fig:bacer_ratio}}
\subfigure[ENZYMES, 4 layers] {\includegraphics[width=0.245\textwidth]{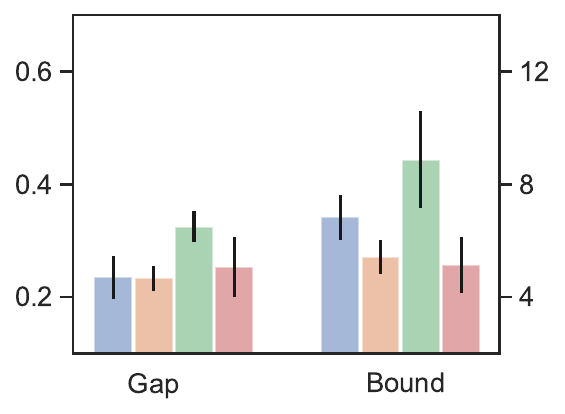}\label{fig:enzymes-4}}
\subfigure[ENZYMES, 6 layers] {\includegraphics[width=0.245\textwidth]{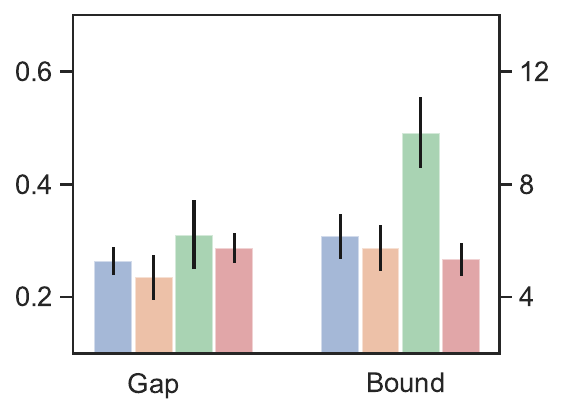}\label{fig:enzymes-6}}
\caption{Loss gaps and bounds}
\label{fig:patterns_gap}
\end{figure}

\begin{figure}[t!]
\centering
\subfigure[Loss gaps and bounds]{\vstretch{1.5}{\includegraphics[width=0.48\textwidth]{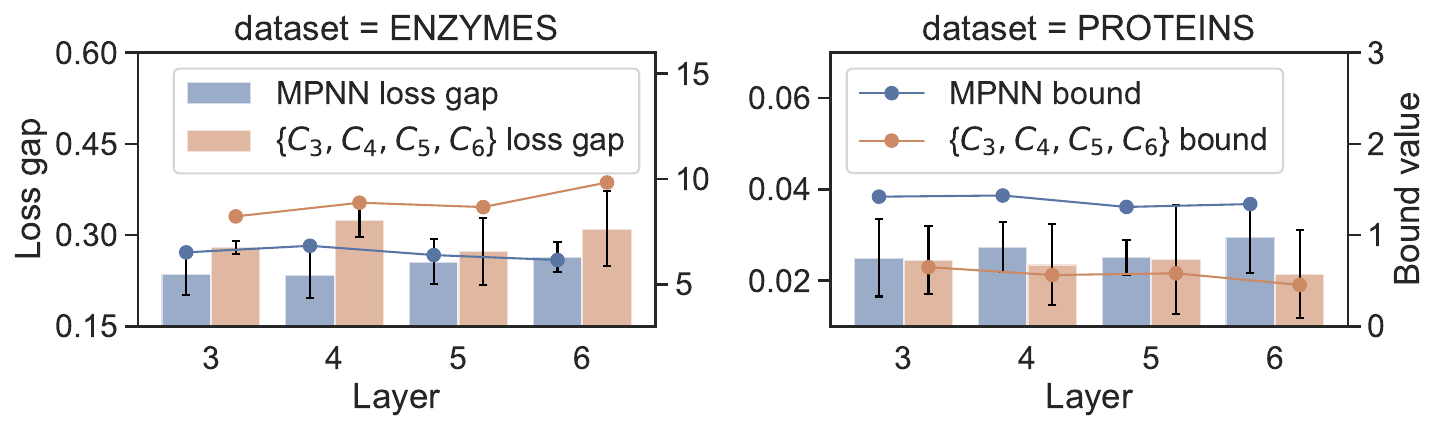}\label{subfig:gaps_layer}}}
\subfigure[1-Wasserstein distance and $\Lip(f)$]{\vstretch{1.5}{\includegraphics[width=0.51\textwidth]{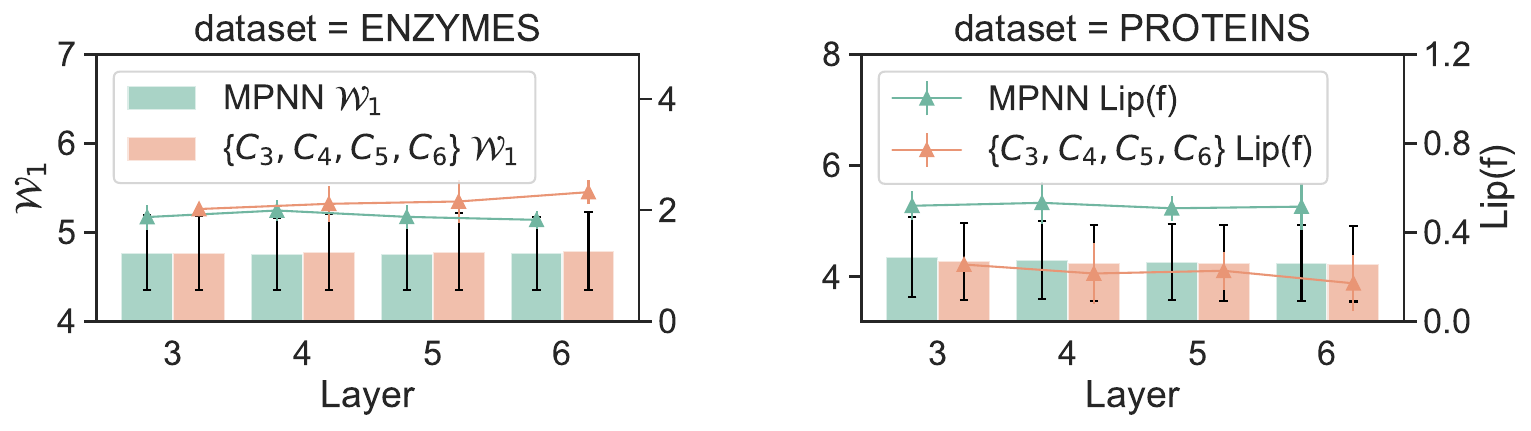}\label{subfig:ratio_layers}}}
\caption{\label{fig:bound_vs_layer}Bound of MPNN and $C_3$-MPNN}%
\end{figure}

\textbf{Why does more expressive power sometimes lead to better generalization?} 
In \cref{subfig:gaps_layer}, we observe two contrasting cases where increased expressivity worsens generalization (ENZYMES) and improves it (PROTEINS). To explore this further, we plot the changes of two major factors from the proposed bound in \cref{lemma:1w_sample_bound}: the 1-Wasserstein distance and $\Lip(f)$, both shown in \cref{subfig:gaps_layer}.
The 1-Wasserstein distance ($\Wc_1$) is computed as:
$
\frac{1}{n} \sum_{j=1}^n \mathcal{W}_1\bigl(\fa_\sharp(\mu_{c,T^j}), \fa_\sharp({\mu}_{c,\tilde{T}^j})\bigr)
$
averaged over all graph classes. We plot these factors over four layers for both MPNN and $\{C_3,C_4,C_5,C_6\}$-MPNN. 
We observe that the inclusion of homomorphism counts worsens generalization in ENZYMES but improves it in PROTEINS. From \cref{subfig:gaps_layer}, this can be attributed to the joint influence of the Wasserstein distance and $\Lip(f)$. In ENZYMES, both the 1-Wasserstein distance and $\Lip(f)$ increase slightly when homomorphism counts are added. 
While this additional expressivity leads to better separation between graphs, in ENZYMES, this increased separation hinders the ability to achieve good concentration within each graph class, ultimately worsening generalization. 
In contrast, for PROTEINS, although the inclusion of homomorphism counts leads to greater graph separation, it also slightly reduces the 1-Wasserstein distance within each class, allowing for better concentration. This improved separation significantly reduces $\Lip(f)$, resulting in enhanced generalization.

\textbf{Can generalization be improved by controlling the Lipschitz constants?} 
Last but not least, since $\Lip(f)$ plays a crucial role in the proposed bound, we aim to investigate whether controlling $\Lip(f)$ can serve as an effective strategy to enhance generalization. A straightforward approach to control $\Lip(f)$ is through normalization techniques. As demonstrated earlier, normalization effectively bounds the diameter of $\phi_\sharp(\mu)$, which, in turn, constrains the encoder's boundedness and subsequently $\Lip(f)$.
To test this, we apply $l_1$-normalisation in the last layer of the MPNN. See \cref{tab:graph_layer_gap_norm} for results. It is evident that normalization reduces the generalization gap across all datasets. This improvement is also reflected in the computed bounds.
Interestingly, the least improvement is observed in the SIDER dataset, where $\Lip(f)$ is already relatively small, and the embeddings are well-concentrated even before normalization. This suggests that the impact of normalization is more pronounced when $\Lip(f)$ is large or when the embeddings are not already well-concentrated.\looseness=-1

\begin{table}[t!]
    \caption{Graph classification gaps with different numbers of MPNN layers. The MPNN embeddings are normalized.}
    \label{tab:graph_layer_gap_norm}
    \centering
    \resizebox{0.8\textwidth}{!}{
    \begin{tabular}{clrrrrr}
    \toprule
    \multirow{2}{*}{} &  & \multicolumn{4}{c}{Dataset} \\ \cmidrule{3-7}
     \# Layers&  & ENZYMES & PROTEINS & MUTAG & SIDER & BACE \\
     \midrule
     \multirow{4}{*}{1} & Acc. gap & 14.67\sd{2.85}&-1.74\sd{0.97}&-8.82\sd{1.65}&1.69\sd{1.77}&-0.34\sd{1.26} \\
     & Loss gap & 0.105\sd{0.010}&-0.018\sd{0.009}&-0.091\sd{0.017}&0.013\sd{0.013}&-0.004\sd{0.010} \\
     & Our Bound & 0.800\sd{0.095}&2.203\sd{0.134}&1.101\sd{0.063}&1.137\sd{0.552}&1.147\sd{0.143} \\
     \midrule
     \multirow{3}{*}{2} & Acc. gap & 14.85\sd{3.56}&-1.84\sd{1.25}&-9.53\sd{1.98}&1.41\sd{2.19}&0.16\sd{1.11} \\
     & Loss gap & 0.098\sd{0.022}&-0.023\sd{0.011}&-0.097\sd{0.019}&0.015\sd{0.006}&0.000\sd{0.010} \\
     & Our Bound & 0.586\sd{0.036}&1.016\sd{0.035}&1.208\sd{0.046}&1.017\sd{0.644}&1.089\sd{0.135} \\
     \midrule
     \multirow{3}{*}{3} & Acc. gap & 19.70\sd{3.74}&-2.72\sd{1.27}&-7.88\sd{0.79}&3.71\sd{1.50}&-0.55\sd{1.67} \\
     & Loss gap & 0.118\sd{0.023}&-0.027\sd{0.011}&-0.083\sd{0.006}&0.030\sd{0.008}&-0.006\sd{0.015} \\
     & Our Bound & 0.572\sd{0.024}&0.834\sd{0.015}&0.993\sd{0.039}&1.221\sd{0.957}&1.167\sd{0.610} \\
     \midrule
     \multirow{3}{*}{4} & Acc. gap & 20.78\sd{2.18}&-0.45\sd{0.68}&-8.23\sd{1.14}&2.98\sd{1.61}&0.28\sd{2.42} \\
     & Loss gap & 0.129\sd{0.007}&-0.004\sd{0.005}&-0.087\sd{0.011}&0.026\sd{0.015}&0.001\sd{0.024} \\
     & Our Bound & 0.573\sd{0.027}&0.847\sd{0.027}&0.848\sd{0.085}&1.039\sd{0.898}&0.705\sd{0.026} \\
     \midrule
     \multirow{3}{*}{5} & Acc. gap & 24.70\sd{3.08}&0.55\sd{3.69}&-8.47\sd{1.26}&1.06\sd{4.64}&0.07\sd{1.85} \\
     & Loss gap & 0.169\sd{0.014}&0.003\sd{0.035}&-0.086\sd{0.013}&0.006\sd{0.038}&0.002\sd{0.014} \\
     & Our Bound & 0.575\sd{0.039}&0.713\sd{0.246}&0.799\sd{0.051}&0.923\sd{0.438}&0.703\sd{0.012} \\
     \midrule
     \multirow{3}{*}{6} & Acc. gap & 23.33\sd{2.75}&0.08\sd{3.31}&-10.24\sd{0.87}&3.44\sd{0.91}&-1.17\sd{1.55}\\
     & Loss gap & 0.169\sd{0.023}&-0.002\sd{0.032}&-0.104\sd{0.008}&0.029\sd{0.009}&-0.013\sd{0.015} \\
     & Our Bound & 0.603\sd{0.032}&0.793\sd{0.136}&0.778\sd{0.049}&1.192\sd{0.561}&0.679\sd{0.018} \\
     \bottomrule
    \end{tabular}%
    }

\end{table}

\section{Conclusion and limitations}
\label{sec:con}
In this work, we examine the generalization of GNNs from a margin-based perspective, based on the work by \citet{Chuang2021-ik}. The bounds use $1$-variance and optimal transport to analyze graph embeddings. We establish a relationship between generalization and the expressive capacity of GNNs, deriving a generalization bound that demonstrates how well-clustered embeddings and separable classes lead to improved generalization. Through case studies on a real-world dataset, we empirically validate these theoretical findings. We also apply empirical sample-based bounds to graph classification tasks, confirming that our theoretical results align with empirical evidence. Our work enables analyzing the generalization of graph encoders through their bounded expressive power.

Nonetheless, our work has some limitations. While we validate the framework on real-world datasets, further large-scale studies across a wider range of datasets and applications are needed to fully establish the proposed approach's general applicability.

\subsection*{Acknowledgements}
This research was supported partially by the Australian Government through the Australian Research Council's Discovery Projects funding scheme (project DP210102273).

\clearpage
\bibliographystyle{plainnat}
\bibliography{references}  

\appendix
\section{Additional related work}\label{sec:apprelwork}
We provide additional references related to the expressiveness of Graph Neural Networks (GNNs). The connection with the Weisfeiler-Leman ($\WL$) test has led to the development of high-order GNNs that surpass $\WL$ and are bounded by the $k$-dimensional Weisfeiler-Leman test ($\WLF{k}$)~\citep{morris2019weisfeiler,maron2019provably,morris2020weisfeiler,GeertsR22}. The method by \citet{morris2019weisfeiler,morris2020weisfeiler} is strictly weaker than $\WLF{k}$, whereas the method by \citet{maron2019provably} can match the expressiveness of $\WLF{k}$. However, these higher-order GNNs incur significant computational costs, rendering them impractical for large-scale datasets.

Incorporating substructure counts has been shown to be an effective strategy for enhancing GNN expressivity beyond $\WL$~\citep{Bouritsas2023-hj,Barcelo2021-rs}. \citet{Bouritsas2023-hj} integrate isomorphism counts of small subgraph patterns into the node and edge features of graphs, while \citet{Barcelo2021-rs} employ a similar approach using homomorphism counts. Building on this concept, \citet{thiede21auto} implemented convolutions on automorphism groups of subgraph patterns. Rather than directly using subgraph counts, \citet{wijesinghe022, Wang_undated-iz} propose integrating local structural information into neighbor aggregation. This approach suggests that the expressivity of the model increases with the subgraph pattern size and aggregation radius. \looseness=-1

Taking a different approach, \citet{Nguyen2020hom} explore the use of graph homomorphism counts directly in convolutions without message passing, demonstrating their universality in approximating invariant functions. \citet{Welke_2023-xn} propose combining homomorphism counts with GNN outputs in the final layer to improve expressivity. Additionally, \citet{bevilacqua22equivariant} represent graphs as collections of subgraphs derived from a predetermined policy. \citet{zhao22stars} and \citet{zhang21nested} extend this idea by representing graphs with a set of induced subgraphs. These methods are closely related to graph kernel techniques that utilize subgraph patterns~\citep{Shervashidze2011-lw, horvath04, costa10}.

Since the $\mathsf{WL}$-based GNN expressivity hierarchy is inherently coarse and qualitative, \citet{zhang2024-quantitative} propose a homomorphism-based expressivity framework, which enables direct comparisons of expressivity between common GNN models. As $\WL$ and $\WLF{k}$ have equivalent translations in homomorphism embeddings~\citep{Dell2018-bg}, both MPNNs and higher-order GNNs can be expressed using homomorphism representations within this framework. Given that homomorphism embeddings are theoretically isomorphism-complete, this framework offers not only a unified but also a complete description of GNN expressivity.

\section{Proofs of \cref{sec:graphenc}}\label{app:proofsecfour}

\factorization*
\begin{proof}
We define the function
$f:\Zc_\phi\to\Zc_{\phi'}$, as follows. Let $z\in \Zc_\phi$ and $G\in\Gc$ such that $\phi(G)=z$. Then,
define $f(z):=\phi'(G)\in\Zc_{\phi'}$. Observe first that $f$ is well-defined. Indeed, if we take another $G'\in\Gc$ such that $\phi(G')=z$, then $\phi(G)=\phi(G')$ and hence also $\phi'(G)=\phi'(G')=f(z)$ since $\phi\sqsubseteq \phi'$ by assumption.
Clearly, $
f\circ \phi =\phi'$, by definition
\end{proof}

\SB*
\begin{proof}
We have 
$d_{\Zc_{\phi'}}(f(\phi(G),f(\phi(H))=0$ in case
$\phi(G)=\phi(H)$ and otherwise,
when $\phi(G)\neq\phi(H)$ also 
$\phi'(G)\neq\phi'(H)$ and hence
$$
d_{\Zc_{\phi'}}(f(\phi(G),f(\phi(H))=
d_{\Zc_{\phi'}}(\phi'(G),\phi'(H))\leq  B.$$
Since $S\leq d_{\Zc_\phi}(\phi(G),\phi(H))$ we have
$$
d_{\Zc_{\phi'}}(f(\phi(G),f(\phi(H))
\leq  (B/S) d_{\Zc_\phi}(\phi(G),\phi(H)),$$
and thus 
$$d_{\Zc_{\phi'}}(f(z),f(z'))
\leq  (B/S) d_{\Zc_\phi}(z,z')
$$
since the domain of $f$ is the codomain of $\phi$. We may thus conclude that $\Lip(f)\leq \frac{B}{S}$, as desired.
\end{proof}

\Lipschitzfactor*
\begin{proof}
We first show that $f\circ\phi=\phi'$ implies the $f_\sharp\bigl(\phi_\sharp(\mu)\bigr)=\phi'_\sharp(\mu)$ of the corresponding pushforward distribution of any distribution $\mu$ om $\Gc$. Indeed, this simply follows from the definitions. One the one hand, for $I\subseteq\Zc_{\phi'}$
$$\phi_\sharp'(\mu)(I):=\mu\bigl(
\{G\in\Gc\mid \phi'(G)\in I\}
\bigr).$$
On the other hand,
\begin{align*}
f_\sharp\bigl(\phi_\sharp(\mu)\bigr)(I)&=\phi_\sharp(\mu)\bigl(
\{z\in\Zc_\phi\mid f(z)\in I\}
\bigr)\\
&=\mu\bigl(G\in\Gc\mid f(\phi(G))\in I
\bigr).
\end{align*}
The equality then follows from $f\circ\phi=\phi'$.
We assume that $f$ is Lipschitz-continuous with $\Lip(f)<\infty$ (otherwise the inequality is satisfied by default and there is nothing to prove).
We show that
$$
\Wc_1\myl(\phi_\sharp'(\mu), \phi_\sharp'(\nu)\myr) \leq  \operatorname{Lip}(f) \cdot \Wc_1\myl(\phi_\sharp(\mu), \phi_\sharp(\nu)\myr).
$$
Let $L_1(\Zc_\phi)$ be the set of 1-Lipschitz functions on $\Zc_\phi$. We use the Kantorovich-Rubinstein dual form of $\Wc_1$, as follows:
\begin{align*}
\Wc_1\myl(\phi_\sharp(\mu), \phi_\sharp(\nu)\myr) & = \sup_{g\in L_1(\Zc_\phi)}\mathbb{E}_{z\sim \phi_\sharp(\mu)}[g(z)] - \mathbb{E}_{z\sim \phi_\sharp(\nu)}[g(z)] \\ 
&= \sup_{g\in L_1(\Zc_\phi)}\int_{\Zc_\phi} g(z)\; \mathrm{d}(\phi_\sharp(\mu) - \phi_\sharp(\nu))(z).
\end{align*}
Note that if $g\in L_1(\Zc)$ then $\frac{1}{\Lip(f)}f\circ g \in L_1(\Zc)$ as well. Then, using our earlier observation about pushforward distributions,
\allowdisplaybreaks
\begin{align*}
    \Wc_1\myl(\phi_\sharp'(\mu), \phi_\sharp'(\nu)\myr) & = \Wc_1\myl(f_\sharp\bigl(\phi_\sharp(\mu)\bigr), f_\sharp\bigl(\phi_\sharp(\nu)\bigr)\myr)\\
    & = \sup_{g\in L_1(\Zc_{\phi'})}\int_{\Zc_{\phi'}} g(z)\: \mathrm{d}\myl(f_\sharp\bigl(\lambda_\sharp(\mu)\bigr) - f_\sharp\bigl(\lambda_\sharp(\nu)\bigr)\myr)(z)\\
    & = \sup_{g\in L_1(\Zc_{\phi'})}\int_{\Zc_{\phi'}} g(z)\: \mathrm{d}f_\sharp(\phi_\sharp(\mu) - \phi_\sharp(\nu))(z)\\
    & = \sup_{g\in L_1(\Zc_{\phi'})}\int_{\Zc_{\phi'}} g\circ f(z)\:\mathrm{d}(\phi_\sharp(\mu) - \phi_\sharp(\nu))(z)\\
    & = \Lip(f) \sup_{g\in L_1(\Zc_{\phi'})}\int_{\Zc_{\phi'}} \frac{g\circ f(z)}{\Lip(f)} \mathrm{d}(\mu - \nu)(z)\\
    & \leq \Lip(f) \sup_{h\in L_1(\Zc_{\phi})}\int_{\Zc_{\phi}} h(x)\: \mathrm{d}(\mu - \nu)(z)\\
    & = \Lip(f) \cdot \Wc_1\myl(\phi_\sharp(\mu), \phi_\sharp(\nu)\myr),
\end{align*} 
as desired.
\end{proof}

\section{Proofs and details of \cref{sec:genal}}\label{app:secfive}

We start by restating Theorem 2 from \citet{Chuang2021-ik} using encoders $\phi$ from some general set $\Xc$ to $\Zc$.

\begin{theorem}[Theorem 2 in \citet{Chuang2021-ik}]
Fix \(\gamma > 0\) and an encoder \(\phi:\Xc\to\Zc\).  Then, for every distribution $\mu$ on $\Xc\times\Yc$, for every predictor $\psi=(\psi_y)_{i\in\Yc}$ and every $\delta\in(0,1)$, with probability at least $1-\delta$ over all choices of $\Sc\sim \mu^m$, we have that the generalization gap $R_\mu(\psi \circ \phi) - \hat{R}_{\gamma, \Sc}(\psi \circ \phi)$ is upper bounded by
\begin{equation*}  
\mathbb{E}_{c \sim \mu_y} \left[ \frac{\operatorname{Lip}\left(\rho_\psi(\cdot, c)\right)}{\gamma} 
        \mathbb{E}_{T,\tilde{T} \sim \mu_{c}^{m_c}} \left[ 
        \mathcal{W}_1\bigl(\phi_\sharp(\mu_{c,T}), \phi_\sharp(\mu_{c,\tilde{T}})\bigr) 
         \right]
    \right]  + \sqrt{\frac{\log (1 / \delta)}{2m}}, 
\end{equation*}
where for each $c\in\Yc$, $m_c$ denotes the number of pairs $(X,c)$ in $\Sc$.  Also, recall that for $T\sim \mu_c^{m_c}$,
$\mu_{c,T}$ is the empirical distribution $\mu_{c,T}:=\sum_{X\in T} \delta_{X}$; similarly for $\mu_{c,\tilde T}$.
\end{theorem}

To obtain \cref{lemma:1w_expectation_bound} we replace $\Xc$ by $\Gc$ and consider graph encoders $\phi:\Gc\to \Zc_\phi$ and $\lambda:\Gc\to\Zc_\lambda$ such that $\lambda$ upper bounds $\phi$ in expressive power. Then, \cref{lem:factor} ensures the existence of $f$ such that $\phi=f\circ\lambda$ and \cref{prop:wasserineq} consequently implies
$
\mathcal{W}_1\myl(\phi_\sharp'(\mu_{c,T}), \phi_\sharp'(\mu_{c,\tilde T})\myr) \leq \operatorname{Lip}(f) \cdot \mathcal{W}_1\myl(\phi_\sharp(\mu_{c,T}), \phi_\sharp(\mu_{c,\tilde T})\myr)$ for any $T,\tilde T\sum \mu_c^{m_c}$. Plugging this into the bound above results in the bound given in \cref{lemma:1w_expectation_bound}.

While the bound in \cref{lemma:1w_expectation_bound} is theoretically useful, the expectation term over \(T,\tilde{T} \sim \mu_c^{m_c}\) is intractable in general. To address this drawback, we derive another bound in \cref{lemma:1w_sample_bound}, which can be computed via sampling in practice and is the one used in our experiments.

\begin{theorem}
\label{lemma:1w_sample_bound}
Let \(\{T^j, \tilde{T}^j\}_{j=1}^n\) be \(n\) pairs of graph samples where each \(T^j, \tilde{T}^j \sim \mu^{\lfloor m_c/2n \rfloor}_{c}\), \(m = \sum_{c=1}^K \lfloor m_c / 2n \rfloor\), 
and \(\Delta(\cdot)\) be the diameter of a space. For any Lipschitz continuous function \(f: \Zc_{\phi} \to \Zc_{\lambda}\) such that \(\phi = f \circ \fa\), with probability at least \(1 - \delta\) for samples $\Sc\sim\mu^m$, we have 
    \begin{align*}
        & R_\mu(\psi \circ \phi) - \hat{R}_{\gamma, \Sc}(\psi \circ \phi) \leq 
        \sqrt{\frac{\log (2 / \delta)}{2m}} + \\
        & \mathbb{E}_{c \sim \sigma} \left[
        \frac{\operatorname{Lip}\left(\rho_\psi(\cdot, c)\right)\operatorname{Lip}(f)}{\gamma}            
        \left(
            \frac{1}{n} \sum_{j=1}^n 
            \mathcal{W}_1\bigl(\fa_\sharp(\mu_{c,T^j}), \fa_\sharp({\mu}_{c,\tilde{T}^j})\bigr)
            + 2\Delta(\lambda_\sharp(\mu_c))\sqrt{\frac{\log(2K / \delta)}{n \lfloor m_c / 2n \rfloor}}
        \right)
        \right].
    \end{align*}
\end{theorem}
The proof is again a consequence of \cref{lem:factor} and \cref{prop:wasserineq}, but this time relying on Corollary 6 in \citet{Chuang2021-ik}. We note that the diameter will be bounded when $B$-bounded graph encoders are considered.

\end{document}